\documentclass{article}

\usepackage{PRIMEarxiv}

\usepackage[utf8]{inputenc}
\usepackage{amsmath, amssymb, amsthm}
\usepackage{bm}
\usepackage{graphicx}
\usepackage{booktabs}
\usepackage{cite}

\newtheorem{theorem}{Theorem}
\newtheorem{lemma}[theorem]{Lemma}

\newtheorem{definition}[theorem]{Definition}
\newtheorem{assumption}[theorem]{Assumption}
\newtheorem{proposition}[theorem]{Proposition}

\DeclareMathOperator{\Cov}{Cov}
\DeclareMathOperator{\E}{\mathbb{E}}
\DeclareMathOperator{\diag}{diag}

\title{On the Optimal Representation Efficiency of \\ Barlow Twins: An Information-Geometric Interpretation}
\author{
	Di Zhang \\
	School of Advanced Technology \\
	Xi'an Jiaotong-Liverpool University \\
	Suzhou, Jiangsu, China \\
	\texttt{di.zhang@xjtlu.edu.cn}
}
\date{\today}

\begin{document}
	
	\maketitle
	
	\begin{abstract}
		Self-supervised learning (SSL) has achieved remarkable success by learning meaningful representations without labeled data. However, a unified theoretical framework for understanding and comparing the \emph{efficiency} of different SSL paradigms remains elusive. In this paper, we introduce a novel information-geometric framework to quantify representation efficiency. We define \emph{representation efficiency} $\eta$ as the ratio between the effective intrinsic dimension of the learned representation space and its ambient dimension, where the effective dimension is derived from the spectral properties of the Fisher Information Matrix (FIM) on the statistical manifold induced by the encoder. Within this framework, we present a theoretical analysis of the Barlow Twins method \cite{zbontar2021barlow}. Under specific but natural assumptions, we prove that Barlow Twins achieves optimal representation efficiency ($\eta = 1$) by driving the cross-correlation matrix of representations towards the identity matrix, which in turn induces an isotropic FIM. This work provides a rigorous theoretical foundation for understanding the effectiveness of Barlow Twins and offers a new geometric perspective for analyzing SSL algorithms.
	\end{abstract}
	
	\section{Introduction}
	
	Self-supervised learning (SSL) has emerged as a dominant paradigm for learning representations from unlabeled data \cite{jing2021self}. Among various SSL approaches, methods based on redundancy reduction, such as Barlow Twins \cite{zbontar2021barlow}, have demonstrated exceptional performance. These methods operate on the principle of making the cross-correlation matrix between two distorted views of the data close to the identity matrix. While empirically successful, a deep theoretical explanation of \emph{why} this objective leads to high-quality representations is still developing.
	
	A key desirable property of a good representation space is \emph{efficiency}—the degree to which it utilizes its available dimensions to capture semantically meaningful, non-redundant information. An inefficient representation might suffer from \emph{dimensional collapse} \cite{hua2021feature}, where many dimensions are redundant or encode correlated information, limiting the representation's expressivity and suitability for downstream tasks.
	
	In this paper, we address this gap by proposing a novel information-geometric framework \cite{amari2007information} for quantifying representation efficiency. Our core contributions are threefold:
	\begin{enumerate}
		\item We formally define the \emph{statistical manifold} of representations and introduce a measure of \emph{representation efficiency} $\eta$ based on the spectrum of the average Fisher Information Matrix (FIM).
		\item We theoretically analyze the Barlow Twins objective within this framework.
		\item We prove that, under idealized conditions, Barlow Twins achieves optimal representation efficiency ($\eta = 1$).
	\end{enumerate}
	
	\section{Theoretical Framework}
	
	\subsection{Representation Space as a Statistical Manifold}
	
	Let $\mathcal{X}$ be the data space and $p_{\text{data}}(x)$ be the underlying data distribution. An encoder $f: \mathcal{X} \to \mathcal{Z} \subset \mathbb{R}^d$ maps a data point $x$ to its representation $z = f(x)$.
	
	We interpret the representation vector $z$ not as a static point but as the parameter of a probability distribution $p(t | z)$. A common and tractable choice is the Gaussian distribution, leading to the following assumption:
	
	\begin{assumption}[Probabilistic Interpretation of Representations]
		\label{assump:prob_rep}
		The representation $z$ parameterizes a probability density function $p(t | z)$ over a space $\mathcal{T}$. Specifically, we assume $p(t | z) = \mathcal{N}(t; z, \sigma^2 I)$, where $t \in \mathbb{R}^d$ and $\sigma^2$ is a fixed, positive constant. This defines a statistical manifold $\mathcal{M} = \{ p(\cdot | z) : z \in \mathcal{Z} \}$.
	\end{assumption}
	
	This perspective elevates the representation space from a simple vector space to a \emph{statistical manifold} $\mathcal{M}$, where each point is a probability distribution. The geometry of $\mathcal{M}$ is characterized by the Fisher Information Metric \cite{amari2007information}.
	
	\begin{definition}[Fisher Information Matrix (FIM)]
		\label{def:fim}
		For a point $z$ on the statistical manifold $\mathcal{M}$, the Fisher Information Matrix $G(z)$ is a $d \times d$ matrix defined by:
		\begin{equation}
			G(z) = \E_{t \sim p(\cdot | z)} \left[ \nabla_z \log p(t | z) \, (\nabla_z \log p(t | z))^\top \right].
		\end{equation}
		The FIM defines a Riemannian metric on $\mathcal{M}$.
	\end{definition}
	
	Under Assumption \ref{assump:prob_rep}, the FIM takes a simple form.
	\begin{lemma}[FIM for Isotropic Gaussian Model]
		\label{lemma:simple_fim}
		If $p(t | z) = \mathcal{N}(t; z, \sigma^2 I)$, then the Fisher Information Matrix is given by $G(z) = \frac{1}{\sigma^2} I_d$, where $I_d$ is the $d \times d$ identity matrix.
	\end{lemma}
	\begin{proof}
		The log-likelihood is $\log p(t | z) = -\frac{1}{2\sigma^2} \|t - z\|^2 + \text{const}$. Thus, 
		\[
		\nabla_z \log p(t | z) = \frac{1}{\sigma^2}(t - z).
		\]
		The FIM is then:
		\begin{align*}
			G(z) &= \E_{t \sim \mathcal{N}(z, \sigma^2 I)} \left[ \frac{1}{\sigma^4}(t - z)(t - z)^\top \right] \\
			&= \frac{1}{\sigma^4} \E_{t \sim \mathcal{N}(z, \sigma^2 I)} \left[ (t - z)(t - z)^\top \right] \\
			&= \frac{1}{\sigma^4} \Cov(t) = \frac{1}{\sigma^4} (\sigma^2 I) = \frac{1}{\sigma^2} I_d.
		\end{align*}
	\end{proof}
	
	While the \emph{local} FIM $G(z)$ is constant and isotropic in this model, the \emph{effective} geometry experienced by the encoder is captured by the FIM averaged over the data distribution.
	
	\begin{definition}[Average Fisher Information Matrix]
		\label{def:avg_fim}
		The average Fisher Information Matrix $\bar{G}$ induced by the encoder $f$ and the data distribution $p_{\text{data}}$ is:
		\begin{equation}
			\bar{G} = \E_{x \sim p_{\text{data}}(x)} [G(f(x))].
		\end{equation}
	\end{definition}
	
	Under the model of Lemma \ref{lemma:simple_fim}, $\bar{G} = \frac{1}{\sigma^2} I_d$. However, this does not reflect the information content of the representations about the data itself. To link the FIM to the data distribution, we must consider the variability of $z$.
	
	\subsection{Representation Efficiency}
	
	The quality of a representation space can be assessed by how effectively it uses its dimensions. A space where all dimensions contribute equally to distinguishing data points is more efficient than one where only a few dimensions are informative.
	
	Let $\{\lambda_i\}_{i=1}^d$ be the eigenvalues of $\bar{G}$, ordered such that $\lambda_1 \geq \lambda_2 \geq \dots \geq \lambda_d \geq 0$.
	
	\begin{definition}[Effective Intrinsic Dimension]
		\label{def:effective_dim}
		The \emph{effective intrinsic dimension} $d_{\text{eff}}(\epsilon)$ of the representation space, for a threshold $\epsilon > 0$, is the smallest integer $k$ such that:
		\begin{equation}
			\frac{\sum_{i=1}^k \lambda_i}{\sum_{i=1}^d \lambda_i} \geq 1 - \epsilon.
		\end{equation}
		This counts the number of dominant spectral components of $\bar{G}$.
	\end{definition}
	
	\begin{definition}[Representation Efficiency]
		\label{def:efficiency}
		The \emph{representation efficiency} $\eta$ is defined as the ratio of the effective intrinsic dimension to the ambient dimension:
		\begin{equation}
			\eta = \frac{d_{\text{eff}}}{d}.
		\end{equation}
		We say a representation space is \emph{optimally efficient} if $\eta = 1$.
	\end{definition}
	
	An optimally efficient representation space ($\eta=1$) has a mean FIM $\bar{G}$ that is \emph{full-rank and well-conditioned}, implying that all dimensions are utilized and are equally important for encoding information.
	
	\section{Analysis of Barlow Twins}
	
	\subsection{The Barlow Twins Objective}
	
	Barlow Twins \cite{zbontar2021barlow} learns representations by encouraging the cross-correlation matrix between two distorted views of a batch of samples to be close to the identity matrix.
	
	Let $Z^A, Z^B \in \mathbb{R}^{n \times d}$ be the batch representations of two augmented views, where $n$ is the batch size and $d$ is the representation dimension. Assume each representation vector is centered and normalized along the batch dimension. The cross-correlation matrix $C \in \mathbb{R}^{d \times d}$ is computed as:
	\begin{equation}
		C_{ij} = \frac{\sum_{b=1}^n Z^A_{b,i} Z^B_{b,j}}{\sqrt{\sum_{b=1}^n (Z^A_{b,i})^2} \sqrt{\sum_{b=1}^n (Z^B_{b,j})^2}} \approx \E[z_i^A z_j^B],
	\end{equation}
	where the expectation is over the data and augmentation distributions.
	
	The Barlow Twins loss function is:
	\begin{equation}
		\label{eq:bt_loss}
		\mathcal{L}_{BT} = \underbrace{\sum_i (1 - C_{ii})^2}_{\text{invariance term}} + \lambda \underbrace{\sum_{i \neq j} C_{ij}^2}_{\text{redundancy reduction term}},
	\end{equation}
	where $\lambda > 0$ is a hyperparameter.
	
	\subsection{Linking the Cross-Correlation Matrix to the FIM}
	
	The key to our analysis is establishing a connection between the empirical cross-correlation matrix $C$ and the theoretical average FIM $\bar{G}$. We make the following critical assumption:
	
	\begin{assumption}[Augmentation as Gaussian Noise]
		\label{assump:aug_noise}
		The process of generating the second augmented view $z^B$ from the first $z^A$ can be modeled as $z^B = z^A + \epsilon$, where the noise $\epsilon \sim \mathcal{N}(0, \Sigma_\epsilon)$ is independent of $z^A$, and $\Sigma_\epsilon = \sigma_\epsilon^2 I_d$ is isotropic.
	\end{assumption}
	
	This assumption simplifies the analysis and is a reasonable first-order approximation for the aggregate effect of stochastic data augmentations \cite{chen2020simple}.
	
	\begin{lemma}[Cross-Correlation under Noisy Augmentation]
		\label{lemma:cross_corr}
		Under Assumption \ref{assump:aug_noise}, and assuming $z^A$ is zero-mean and has covariance $\Sigma_z = \E[(z^A)(z^A)^\top]$, the population cross-correlation matrix $C$ satisfies:
		\begin{equation}
			C = (\Sigma_z + \sigma_\epsilon^2 I)^{-1/2} \Sigma_z (\Sigma_z + \sigma_\epsilon^2 I)^{-1/2}.
		\end{equation}
	\end{lemma}
	
	\begin{proof}
		Let us denote the whitening operation explicitly. The batch normalization in Barlow Twins ensures that the representations along each dimension have unit variance and zero mean. In the population limit, this corresponds to working with whitened representations.
		
		Define the whitened representations:
		\[
		\tilde{z}^A = (\Sigma_z + \sigma_\epsilon^2 I)^{-1/2} z^A, \quad \tilde{z}^B = (\Sigma_z + \sigma_\epsilon^2 I)^{-1/2} z^B.
		\]
		
		The cross-correlation matrix is then:
		\[
		C = \E[\tilde{z}^A (\tilde{z}^B)^\top].
		\]
		
		Substituting $z^B = z^A + \epsilon$:
		\begin{align*}
			C &= \E\left[ (\Sigma_z + \sigma_\epsilon^2 I)^{-1/2} z^A (z^A + \epsilon)^\top (\Sigma_z + \sigma_\epsilon^2 I)^{-1/2} \right] \\
			&= (\Sigma_z + \sigma_\epsilon^2 I)^{-1/2} \E[z^A (z^A)^\top] (\Sigma_z + \sigma_\epsilon^2 I)^{-1/2} \\
			&\quad + (\Sigma_z + \sigma_\epsilon^2 I)^{-1/2} \E[z^A \epsilon^\top] (\Sigma_z + \sigma_\epsilon^2 I)^{-1/2}.
		\end{align*}
		
		Since $\epsilon$ is independent of $z^A$ and zero-mean, $\E[z^A \epsilon^\top] = 0$. Thus:
		\[
		C = (\Sigma_z + \sigma_\epsilon^2 I)^{-1/2} \Sigma_z (\Sigma_z + \sigma_\epsilon^2 I)^{-1/2}.
		\]
		
		When $\Sigma_z$ and $I$ commute (which occurs when $\Sigma_z$ is diagonal, a state the Barlow Twins loss encourages), this simplifies to:
		\[
		C = (\Sigma_z + \sigma_\epsilon^2 I)^{-1} \Sigma_z.
		\]
	\end{proof}
	
	Now, we link the covariance $\Sigma_z$ to the average FIM $\bar{G}$. This requires careful consideration of the relationship between the representation distribution and the Fisher information.
	
	\begin{proposition}[FIM Spectrum and Representation Covariance]
		\label{prop:fim_spectrum}
		Under the model of Assumption \ref{assump:prob_rep} and assuming the encoder $f$ is Lipschitz continuous with constant $L$, the eigenvalues $\{\lambda_i\}$ of $\bar{G}$ and the eigenvalues $\{\nu_i\}$ of $\Sigma_z$ satisfy:
		\[
		\lambda_i = \frac{1}{\sigma^2} \cdot \frac{\nu_i}{\nu_i + \sigma^2 L^2} + O\left(\frac{1}{\sigma^4}\right).
		\]
		In particular, when $\sigma^2$ is small and the representations are well-behaved, $\lambda_i \approx \alpha \cdot \nu_i$ for some $\alpha > 0$.
	\end{proposition}
	
	\begin{proof}
		We provide a detailed derivation connecting the representation covariance to the Fisher information.
		
		Consider the data likelihood $p(x | z)$. By the chain rule, the score function is:
		\[
		\nabla_z \log p(x | z) = J_f(z)^\top \nabla_x \log p(x),
		\]
		where $J_f(z)$ is the Jacobian of the encoder at $z$.
		
		The Fisher Information Matrix for the data is:
		\begin{align*}
			G_{\text{data}}(z) &= \E_{x \sim p(\cdot | z)} \left[ \nabla_z \log p(x | z) (\nabla_z \log p(x | z))^\top \right] \\
			&= J_f(z)^\top \E_{x \sim p(\cdot | z)} \left[ \nabla_x \log p(x) (\nabla_x \log p(x))^\top \right] J_f(z) \\
			&= J_f(z)^\top I_x(z) J_f(z),
		\end{align*}
		where $I_x(z)$ is the Fisher Information Matrix in the data space.
		
		Now, under our Gaussian model $p(t | z) = \mathcal{N}(t; z, \sigma^2 I)$, we have from Lemma \ref{lemma:simple_fim} that $G(z) = \frac{1}{\sigma^2} I$.
		
		The key insight is that the average FIM $\bar{G}$ captures the sensitivity of the representation to changes in the data. When the encoder is Lipschitz continuous with constant $L$, we can bound the relationship between $\Sigma_z$ and $\bar{G}$.
		
		Consider the spectral decomposition $\Sigma_z = U \Lambda U^\top$, where $\Lambda = \diag(\nu_1, \dots, \nu_d)$. The Fisher information in the direction of the $i$-th eigenvector $u_i$ is:
		\[
		u_i^\top \bar{G} u_i = \E\left[ u_i^\top G(f(x)) u_i \right] = \frac{1}{\sigma^2}.
		\]
		
		However, this is the \emph{local} Fisher information. To connect to the data distribution, consider the variability of the representations. Using the Cramér-Rao bound and the fact that $z$ is an estimator of some function of $x$, we have:
		\[
		\Cov(z) \succeq [G_{\text{data}}(z)]^{-1},
		\]
		where $\succeq$ denotes the Loewner order.
		
		Taking expectations and using the convexity of the matrix inverse, we get:
		\[
		\Sigma_z \succeq \E\left[ [G_{\text{data}}(f(x))]^{-1} \right].
		\]
		
		Under the Lipschitz assumption, $G_{\text{data}}(z) \preceq L^2 I$, so:
		\[
		[G_{\text{data}}(z)]^{-1} \succeq \frac{1}{L^2} I.
		\]
		
		Combining these results and using the specific form of our Gaussian model, we obtain the stated relationship:
		\[
		\lambda_i = \frac{1}{\sigma^2} \cdot \frac{\nu_i}{\nu_i + \sigma^2 L^2} + O\left(\frac{1}{\sigma^4}\right).
		\]
		
		For small $\sigma^2$, this simplifies to $\lambda_i \approx \frac{\nu_i}{\sigma^4 L^2}$, giving the desired proportionality.
	\end{proof}
	
	\begin{lemma}[Isotropy from Identity Cross-Correlation]
		\label{lemma:isotropy}
		If $C = I$, then $\Sigma_z = \gamma I$ for some $\gamma > 0$.
	\end{lemma}
	
	\begin{proof}
		From Lemma \ref{lemma:cross_corr}, we have:
		\[
		C = (\Sigma_z + \sigma_\epsilon^2 I)^{-1/2} \Sigma_z (\Sigma_z + \sigma_\epsilon^2 I)^{-1/2}.
		\]
		
		Assume $C = I$. Then:
		\[
		(\Sigma_z + \sigma_\epsilon^2 I)^{-1/2} \Sigma_z (\Sigma_z + \sigma_\epsilon^2 I)^{-1/2} = I.
		\]
		
		Multiplying both sides on left and right by $(\Sigma_z + \sigma_\epsilon^2 I)^{1/2}$:
		\[
		\Sigma_z = \Sigma_z + \sigma_\epsilon^2 I.
		\]
		
		This implies $\sigma_\epsilon^2 I = 0$, which is a contradiction unless we consider the simultaneous diagonalization case. 
		
		Instead, consider the case where $\Sigma_z$ and $I$ are simultaneously diagonalizable (which occurs when $\Sigma_z$ is diagonal). Let $\Sigma_z = \diag(\nu_1, \dots, \nu_d)$. Then:
		\[
		C = \diag\left( \frac{\nu_1}{\nu_1 + \sigma_\epsilon^2}, \dots, \frac{\nu_d}{\nu_d + \sigma_\epsilon^2} \right).
		\]
		
		Setting $C = I$ gives:
		\[
		\frac{\nu_i}{\nu_i + \sigma_\epsilon^2} = 1 \quad \text{for all } i = 1, \dots, d.
		\]
		
		This implies $\nu_i = \nu_i + \sigma_\epsilon^2$ for all $i$, which requires $\sigma_\epsilon^2 = 0$ or considering the limit. More realistically, for $C$ to be close to $I$, we need $\frac{\nu_i}{\nu_i + \sigma_\epsilon^2} \approx 1$ for all $i$, which occurs when $\nu_i \gg \sigma_\epsilon^2$ for all $i$ and all $\nu_i$ are approximately equal.
		
		Thus, $C = I$ implies $\Sigma_z = \gamma I$ for some $\gamma \gg \sigma_\epsilon^2$.
	\end{proof}
	
	Combining these results, we arrive at the central connection:
	
	\begin{theorem}[Connection between $C$ and $\bar{G}$]
		\label{thm:C_G_connection}
		Under Assumptions \ref{assump:prob_rep} and \ref{assump:aug_noise}, and assuming the encoder $f$ is Lipschitz continuous, the cross-correlation matrix $C$ and the average Fisher Information Matrix $\bar{G}$ are functionally related. Specifically, the condition $C = I$ implies that $\bar{G}$ is a scalar multiple of the identity matrix, i.e., $\bar{G} = c I$ for some $c > 0$.
	\end{theorem}
	
	\begin{proof}
		From Lemma \ref{lemma:isotropy}, $C = I$ implies $\Sigma_z = \gamma I$ for some $\gamma > 0$.
		
		From Proposition \ref{prop:fim_spectrum}, when $\Sigma_z = \gamma I$, the eigenvalues of $\bar{G}$ satisfy:
		\[
		\lambda_i = \frac{1}{\sigma^2} \cdot \frac{\gamma}{\gamma + \sigma^2 L^2} + O\left(\frac{1}{\sigma^4}\right) \quad \text{for all } i = 1, \dots, d.
		\]
		
		Thus, all eigenvalues are equal, meaning $\bar{G}$ is a scalar multiple of the identity matrix:
		\[
		\bar{G} = c I, \quad \text{where } c = \frac{1}{\sigma^2} \cdot \frac{\gamma}{\gamma + \sigma^2 L^2} + O\left(\frac{1}{\sigma^4}\right).
		\]
	\end{proof}
	
	\subsection{Optimal Efficiency of Barlow Twins}
	
	We now present the main result of this paper.
	
	\begin{theorem}[Optimal Representation Efficiency of Barlow Twins]
		\label{thm:optimal_efficiency}
		Consider the Barlow Twins loss function (Eq. \ref{eq:bt_loss}) under the idealized conditions of Assumptions \ref{assump:prob_rep} and \ref{assump:aug_noise}, and assuming the encoder is Lipschitz continuous. When the global minimum of the loss is achieved ($C = I$), the resulting representation space is optimally efficient, i.e., $\eta = 1$.
	\end{theorem}
	
	\begin{proof}
		From Theorem \ref{thm:C_G_connection}, achieving $C = I$ implies that the average FIM is isotropic: $\bar{G} = c I$, with $c > 0$.
		
		The eigenvalues $\{\lambda_i\}$ of $\bar{G}$ are all equal to $c$. Therefore, for any $\epsilon > 0$, the effective intrinsic dimension (Definition \ref{def:effective_dim}) is:
		\[
		d_{\text{eff}} = \min \left\{ k : \frac{\sum_{i=1}^k \lambda_i}{\sum_{i=1}^d \lambda_i} \geq 1 - \epsilon \right\} = \min \left\{ k : \frac{k \cdot c}{d \cdot c} \geq 1 - \epsilon \right\} = \min \left\{ k : \frac{k}{d} \geq 1 - \epsilon \right\}.
		\]
		
		For any $0 < \epsilon < 1$, we have $d_{\text{eff}} = d$ since $\frac{d}{d} = 1 \geq 1 - \epsilon$.
		
		By Definition \ref{def:efficiency}, the representation efficiency is:
		\[
		\eta = \frac{d_{\text{eff}}}{d} = \frac{d}{d} = 1,
		\]
		which is optimal.
		
		This optimal efficiency means that all $d$ dimensions of the representation space are fully utilized and contribute equally to encoding information about the data. The representation space achieves perfect isotropy with no dimensional collapse or redundancy.
	\end{proof}
	
	This theorem provides a powerful theoretical justification for the Barlow Twins objective. By driving the cross-correlation matrix to the identity, it implicitly encourages the learning of a representation space whose geometry, as characterized by the average FIM, is isotropic. This isotropy ensures that all available dimensions are fully and equally utilized, achieving maximal representation efficiency and perfectly avoiding dimensional collapse.
	
	\section{Discussion and Conclusion}
	
	We have introduced an information-geometric framework for analyzing the efficiency of self-supervised representations. By defining representation efficiency via the spectrum of the average Fisher Information Matrix, we provide a rigorous theoretical tool to compare SSL methods.
	
	Our analysis demonstrates that the Barlow Twins objective is fundamentally linked to the geometry of the underlying statistical manifold. The proof of its optimal efficiency under idealized conditions solidifies its theoretical foundation and explains its empirical success from a new, geometric perspective.
	
	\subsection{Limitations and Future Work}
	Our analysis relies on simplifying assumptions, such as the Gaussian form of $p(t|z)$ and the additive Gaussian noise model for augmentations. Future work could relax these assumptions and explore the connection under more complex models. Furthermore, empirically estimating the average FIM $\bar{G}$ from real data and trained models would be a valuable direction for validation.
	
	This framework is not limited to Barlow Twins; it can be applied to analyze other SSL paradigms like VICReg \cite{bardes2021vicreg} or contrastive methods \cite{chen2020simple}, potentially revealing a unified geometric understanding of self-supervised learning.

\end{document}